\documentclass[conference]{IEEEtran}
\usepackage{ifpdf}
\usepackage{latexsym,amsfonts,amssymb,amsmath,graphicx,epsf,cite,bbm,float}
\usepackage{ifpdf,flexisym}
\usepackage{epstopdf}
\usepackage{caption}
\usepackage{subcaption}
\usepackage{comment}
\usepackage{amsthm}
\usepackage{mathtools}
\usepackage{bbm}
\usepackage{lettrine}
\usepackage{algorithm}
\usepackage{algorithmic}

\newtheorem{theorem}{Theorem}
\newtheorem{lemma}{Lemma}

\newtheorem{proposition}{Proposition}
\newtheorem{corollary}{Corollary}
\theoremstyle{definition}
\newtheorem{definition}{Definition}

\theoremstyle{remark}

\begin{document}
\title{Interaction Information for Causal Inference:\\ The Case of Directed Triangle}

\author{
\IEEEauthorblockN{AmirEmad Ghassami}
\IEEEauthorblockA{Department of ECE\\
Coordinated Science Laboratory\\
University of Illinois at Urbana-Champaign\\
ghassam2@illinois.edu}
\and
\IEEEauthorblockN{Negar Kiyavash}
\IEEEauthorblockA{Departments of ECE and ISE\\
Coordinated Science Laboratory\\
University of Illinois at Urbana-Champaign\\
kiyavash@illinois.edu}
}

\maketitle

\begin{abstract}
To be considered for the 2017 IEEE Jack Keil Wolf ISIT Student Paper Award. - Interaction information is one of the multivariate generalizations of mutual information, which expresses the amount information shared among a set of variables, beyond the information, which is shared in any proper subset of those variables. Unlike (conditional) mutual information, which is always non-negative, interaction information can be negative. We utilize this property to find the direction of causal influences among variables in a triangle topology under some mild assumptions.
\end{abstract}


\section{Introduction}
\label{sec:intro}

Mutual information is one of the fundamental information-theoretic quantities which measures the co-dependence between two random variables. 
Mutual information could be generalized to the multivariate case in different ways. The most well known generalizations are total correlation \cite{watanabe1960information} (also known as multi-information \cite{studeny1998multiinformation}), and interaction information \cite{mcgill1954multivariate,fano1961book}.
In this work, we focus on interaction information, another information theoretic quantity intimately related to mutual information. This quantity has been studied from different view points and under different names in the literature \cite{mcgill1954multivariate,fano1961book,yeung1991new,bell2003co,jakulin2003quantifying}.
In the case of three random variables, interaction information
is the gain (or loss) in information transmitted between any two of the variables, due to additional knowledge of the third random variable \cite{mcgill1954multivariate}.
That is, interaction information is the difference between the conditional and unconditional mutual information between two of the variables, where the conditioning is on the third variable. It is important to note that unlike (conditional) mutual information which is always non-negative, interaction information can be negative. In fact, this is the property we take advantage of in this study. We will show that the sign of interaction information may be used to identify the direction of influence among variables, a fundamental problem of interest in causal inference. Other information-theoretic quantities such as entropy and directed information have also been proposed to infer causality in appropriate settings \cite{quinn2013efficient, etesami2016learning, quinn2011equivalence, kocaoglu2016entropic, etesami2016learning2}. 

Learning causal relations among variables is a canonical problem in several fields of science such as economics, biology, computer science, etc. In an observational setup, where performing interventions is not possible, the main approach to identify direction of influences is to perform some sort of statistical dependency tests on data \cite{pearl2009causality}. 
The triangle structure comprised of three variables on a cycle of length three is one of the most problematic structures. 
This is because of the fact that dependency tests, which are typically performed to find the directions, all fail in this setting.
In Pearl's language \cite{pearl2009causality}, this is because all triangles are in the same Markov equivalent class. We will show that under certain conditions, using the sign of interaction information, we can uniquely identify the underlying causal influences in a triangle structure.

The rest of the paper is organized as follows: In Section \ref{sec:II} we provide the formal definition of interaction information, as well as some of its properties. In Section \ref{sec:causal}, after introducing the problem of our interest, a discussion regarding the sign of interaction information is provided. In the same section we outline our approach for identifying causal relationships among three variable structures, which could not be identified using merely conventional dependency tests. Our concluding remarks are stated in Section \ref{sec:conc}.

\section{Interaction Information}
\label{sec:II}

The general formula for interaction information for a set of variables $V$ is defined as \cite{bell2003co}
\[
I(V)\coloneqq\sum_{U\subseteq V}(-1)^{|U|+1}H(U),
\]
where $|U|$ denotes the cardinality of the subset $U$ and $H(\cdot)$ is Shannon's entropy function (note that $H(\emptyset)=0$).
Intuitively, interaction information is the amount of information shared by all the variables together.

For the case of three variables, $X$, $Y$ and $Z$:
\begin{align*}
I(X;Y;Z)=&+[H(X)+H(Y)+H(Z)]\\
&-[H(X,Y)+H(X,Z)+H(Y,Z)]\\
&+H(X,Y,Z).
\end{align*}
Here, interaction information could be represented in terms of mutual information as follows \cite{mcgill1954multivariate}:
\begin{proposition}
\label{prop:intinfo}
For the case of three variables, the interaction information could be written as
\begin{align*}
I(X;Y;Z)
&=I(X;Y)-I(X;Y|Z)\\
&=I(X;Z)-I(X;Z|Y)\\
&=I(Y;Z)-I(Y;Z|X).
\end{align*}
\end{proposition}

Using this formulation, one can see that in the case of three variables, interaction information quantifies how much the information shared between two variables defers from what they share if the third variable was known.  Figure \ref{fig:intinfo} depicts a graphical representation of the information-theoretic quantities of our interest.

\begin{figure}[t]
\centering
\includegraphics[scale=0.5]{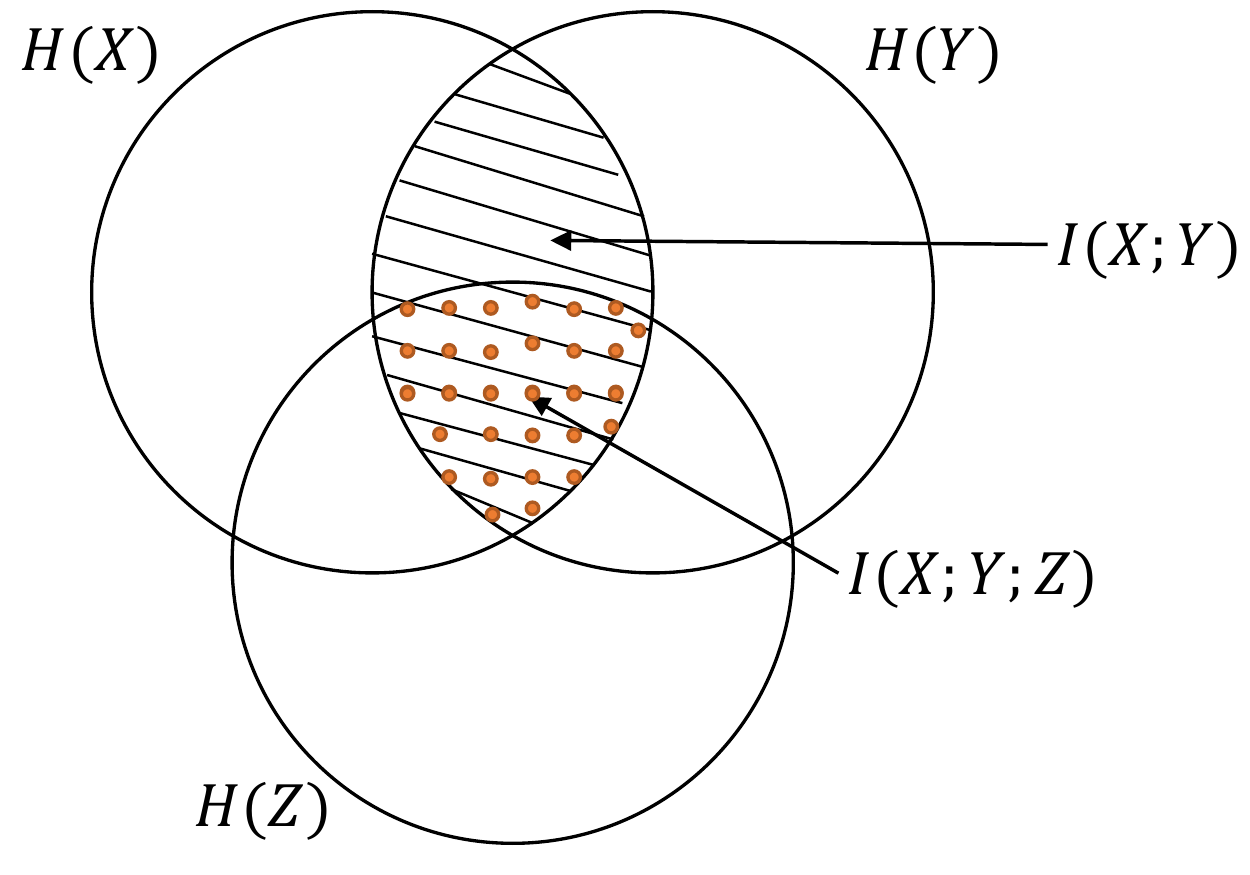}
\caption{Graphical representation of information theoretic quantities.}
\label{fig:intinfo}
\end{figure}

Several properties of interaction information in the case of three variables has been studied in the literature. Specifically, Yeung \cite{yeung1991new} showed that
\begin{align*}
-\min&\{I(X;Y|Z),I(X;Z|Y),I(Y;Z|X)\}\\
&\le I(X;Y;Z)\le\min\{I(X;Y),I(X;Z),I(Y;Z)\}.
\end{align*}
We refer readers to \cite{tsujishita1995triple}, where Tsujishita has provided a more in depth mathematical study of the bounds on interaction information as well as some other properties for this quantity.

We present another property of the interaction information in the following Lemma, which could be proven using the chain rule for mutual information.
\begin{lemma}
\label{lem:prop}
For the case of three variables, the interaction information could be written as
\begin{align*}
I(X;Y;Z)
&=I(X,Y;Z)-I(X;Z|Y)-I(Y;Z|X)\\
&=I(X,Z;Y)-I(X;Y|Z)-I(Z;Y|X)\\
&=I(Y,Z;X)-I(Y;X|Z)-I(Z;X|Y).
\end{align*}
\end{lemma}
Lemma \ref{lem:prop} may be interpreted as follows. Interaction information is the amount of information variables $X$ and $Y$ share with $Z$, minus the information that is shared between $X$ and $Z$ alone and the information shared between $Y$ and $Z$ alone. We will revisit this property in Section \ref{sec:causal}.\\

\section{Application to Causal Inference}
\label{sec:causal}

\subsection{preliminaries} 
\label{subsec:pre}

In this subsection we introduce some definitions and concepts that we require later. Most of the definitions are adopted from \cite{koller2009probabilistic}.
\begin{definition}
a \textit{directed acyclic graph (DAG)} is a finite directed graph with no directed cycles.
\end{definition}
\begin{definition}
A Bayesian network structure $G$ is a DAG whose nodes represent random variables $X_1, . . . , X_n$. Let $PA_{X_i}$ denote the parents of $X_i$ in $G$, and $ND_{X_i}$ denote the variables in the graph that are not descendants of $X_i$. Then $G$ encodes the following set of conditional independence assumptions:
\[
\text{For each variable } X_i: (X_i \perp ND_{X_i} | PA_{X_i} ).
\]
\end{definition}

Bayesian networks are commonly used to represent causal relationships among the set of variables \cite{pearl2009causality,spirtes2000causation}. In such a representation, a directed edge from variable $X$ to variable $Y$ indicates that variable $X$ is a direct cause of variable $Y$. Therefore, a DAG summarizes the causal relationships among the variables.
\begin{definition}
The skeleton of a Bayesian network graph $G$ over the set of variables $V$ is an undirected graph over $V$ that contains an edge $xy$ for every directed edge $\overset{\rightarrow}{xy}$ in $G$.
\end{definition}

We will focus on two skeletons in this work: $P_2$ and triangle, which are paths of length two and cycles of length 3, both on three variables, respectively.

\begin{definition}
A distribution $P$ is faithful to $G$ if $G$ represents all the independency relations contained in $P$.
\end{definition}
Throughout the rest of the paper, we assume the faithfulness assumption on the probability distribution.
\begin{definition}
Two graph structures $G_1$ and $G_2$ over $V$ are Markov equivalent if every probability distribution that is compatible with one of the graphs is also compatible with the other.
The set of all graphs over $V$ is partitioned into a set of mutually exclusive and exhaustive Markov equivalence classes, which are the set of equivalence classes induced by the Markov equivalence relation.
\end{definition}

In Subsection \ref{subset:tri}, we need to be able to quantify the strength of a causal effect, which is an important topic of research on its own in the field of causal inference. For this purpose, we use the results from  \cite{janzing2013quantifying}.

Let $G$ be a DAG on a set of variables $V=\{X_1, X_2,...,X_n\}$. Following \cite{janzing2013quantifying}, we define the strength of the causal influence of a set of arrows $S$ as
\begin{align*}
\mathcal{C}_{S}&\coloneqq D(P\|P_{S}),
\end{align*}
where, $D(\cdot\| \cdot)$ denotes the Kullback-Leibler divergence, $P$ is the joint distribution, and $P_{S}$ is the interventional distribution defined as follows. Set $PA^S_j$ as the set of those parents $X_i$ of $X_j$ for which $(i,j)\in S$ and $PA^{\overline{S}}_j$ as those for which $(i,j)\notin S$. Set 
\[P_S(x_j|pa_j^{\overline{S}})=\sum_{pa_j^S}P(x_j|pa_j^{\overline{S}},pa_j^S)P_{\Pi}(pa_j^S),\] where $P_{\Pi}(pa_j^S)$ for a given $j$ denotes the product of marginal distributions of all variables in $PA^S_j$. The interventional distribution is hence defined as
\[
P_S(x_1,...,x_n)\coloneqq\prod_jP_S(x_j|pa_j^{\overline{S}}).
\]
As an example, in the first DAG in Figure \ref{fig:tri}, we have
\begin{align*}
\mathcal{C}_{X\rightarrow Y}
=\sum_{x,y,z}P(x,y,z)\log\frac{P(y|z,x)}{\sum_{x'}P(y|z,x')P(x')}.
\end{align*}

\subsection{$P_2$ DAG}

\begin{figure}[t]
	\centering
	\includegraphics[scale=0.35]{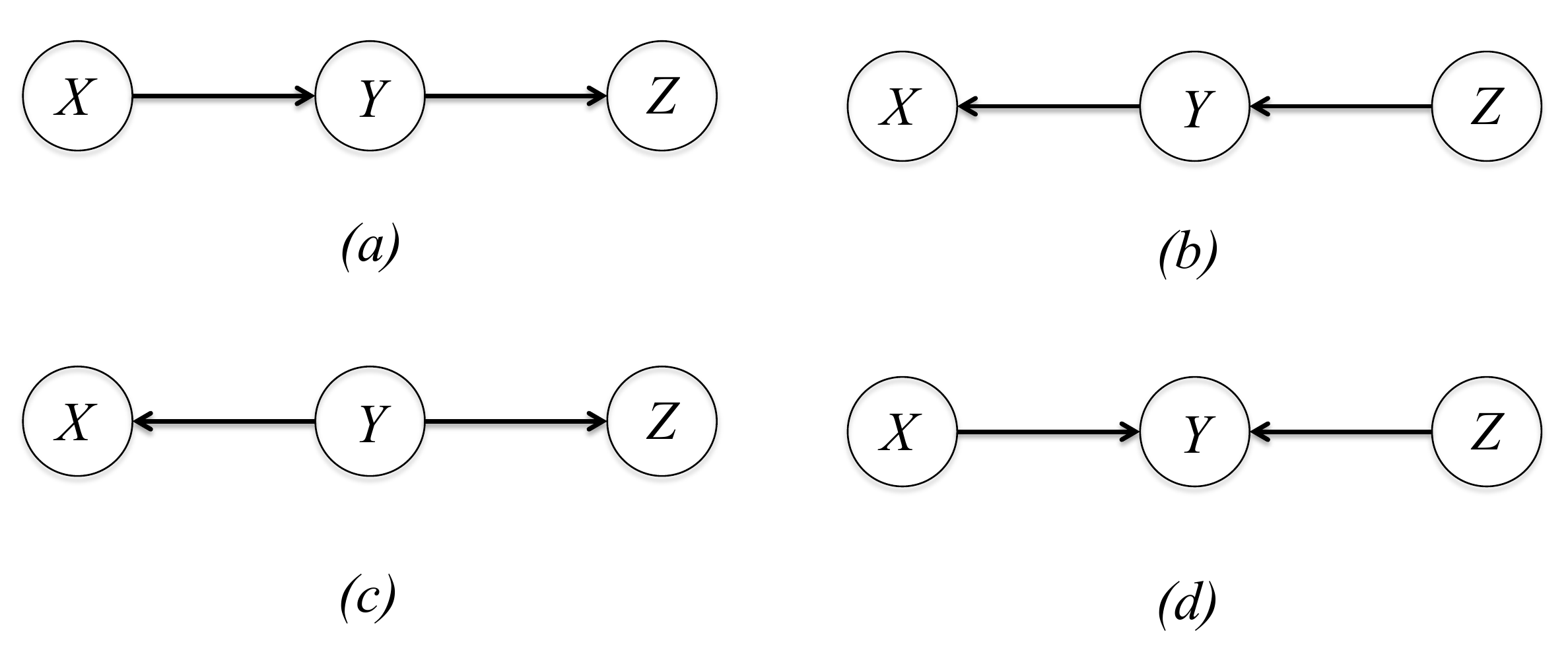}
	\caption{$(a)$, $(b)$: chain structure, $(c)$: fork structure, $(d)$: $v$-structure.}
	\label{fig:v}
\end{figure}

There are 4 possible DAGs on the $P_2$ skeleton for any ordered set of 3 variables $(X,Y,Z)$, where in the skeleton, $Y$ is connected to $X$ and $Z$ (see Figure \ref{fig:v}). The structures in parts $(a)$ and $(b)$ are called a \textit{chain}, $(c)$ is called a \textit{fork}, and $(d)$ is called a \textit{$v$-structure}. In a $v$-structure, the middle variable is called a $collider$.

It is known that in the observational setup, we can identify a DAG at most up to its Markov equivalent class \cite{pearl2009causality}.
Having the information about the skeleton of the DAG (which could be obtained from the correlations), using only dependency tests one can distinguish a $v$-structure from other three: If variables $X$ and $Z$ are dependent given $Y$, but independent when $Y$ is not observed the true structure is a $v$-structure; otherwise, it will be one of the other three. Therefore, for the skeleton $P_2$, the chain structure and the fork structure are in one Markov equivalence class, while the $v$-structure is in a different class.

In the following, we show that determining the correct Markov equivalent class in Figure \ref{fig:v}, could also be performed by calculating the sign of the interaction information, which could be useful from algorithm design point of view.
In general, as evident from Proposition \ref{prop:intinfo}, positive interaction information indicates that each one of the variables partially or completely constitutes the dependency between the other two variables. In Figure \ref{fig:v} parts $(a)$ to $(c)$, given $Y$, variables $X$ and $Z$ are independent. Hence we have \[I(X;Z)\ge0 \text{ ~and~ } I(X;Z|Y)=0,\] which implies that $I(X;Y;Z)\ge0$. On the other hand, negative interaction information indicates that observation of each one of the variables increases the correlation between the other two. In the $v$-structure (Figure \ref{fig:v} part (d)), variables $X$ and $Z$ are independent, but can be dependent conditioned on $Y$. Hence we have \[I(X;Z)=0 \text{ ~and~ }  I(X;Z|Y)\ge0,\] which implies that $I(X;Y;Z)\le0$. Therefore, knowing that the interaction information is positive or negative we can distinguish between the two Markov equivalent classes.

\begin{figure}[t]
	\centering
	\includegraphics[scale=0.5]{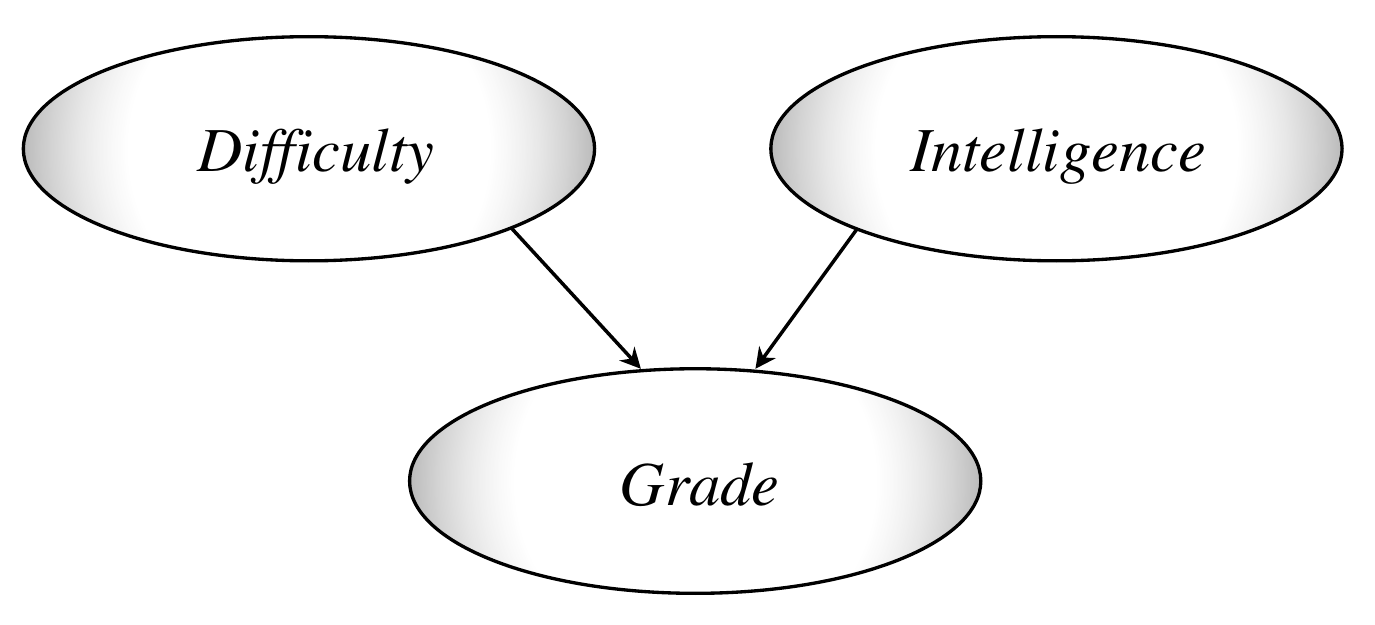}
	\caption{Example of a structure with non-positive interaction information.}
	\label{fig:ex1}
\end{figure}

Note that again in light of Proposition \ref{prop:intinfo}, in the $v$-structure in Figure \ref{fig:v}(d), when $X$ and $Z$ are the common causes of a third variable $Y$, knowing $X$ can increase the correlation between $Z$ and $Y$, a result which may not be intuitive in the first glance.

We use an example from \cite{koller2009probabilistic} to illustrate the case of negative interaction information. Suppose an exam is given to a student. The difficulty of the exam and the intelligence of the student are two independent variables. But, when the student's grade is observed, this new variable correlates the difficulty of the exam and the student's intelligence (see Figure \ref{fig:ex1}). Consider the expression for interaction information represented in Lemma \ref{lem:prop}, with $X=$\textit{Difficulty}, $Y=$\textit{Intelligence}, and $Z=$\textit{Grade}. From Proposition \ref{prop:intinfo}, since $I(X;Y)=0$, it is easy to see that the interaction information is negative. Therefore, the correlation between \textit{Difficulty} and \textit{Grade} when \textit{Intelligence} is observed, and the correlation between \textit{Intelligence} and \textit{Grade} when \textit{Difficulty} is observed are both high and their sum, over calculates the correlation between the pair (\textit{Difficulty}, \textit{Intelligence}) and \textit{Grade}.

\subsection{Triangle DAG}
\label{subset:tri}
 Unlike the structure in Figure \ref{fig:v}, the triangle DAGs, which are DAGs on three variables whose skeleton is a cycle of length 3, are all in the same Markov equivalent class. Therefore, the dependency tests cannot distinguish between graphs with this structure. Nevertheless, triangle DAGs appear in many real-life problems and the ability to reconstruct this structure is of great interest in many fields. Figure \ref{fig:tri} shows all the 6 possible triangle DAGs on the set of variables $\{X,Y,Z\}$. Note that the edge directions must not form a cycle, otherwise the structure will not be a DAG.

\begin{figure}[t]
	\centering
	\includegraphics[scale=0.58]{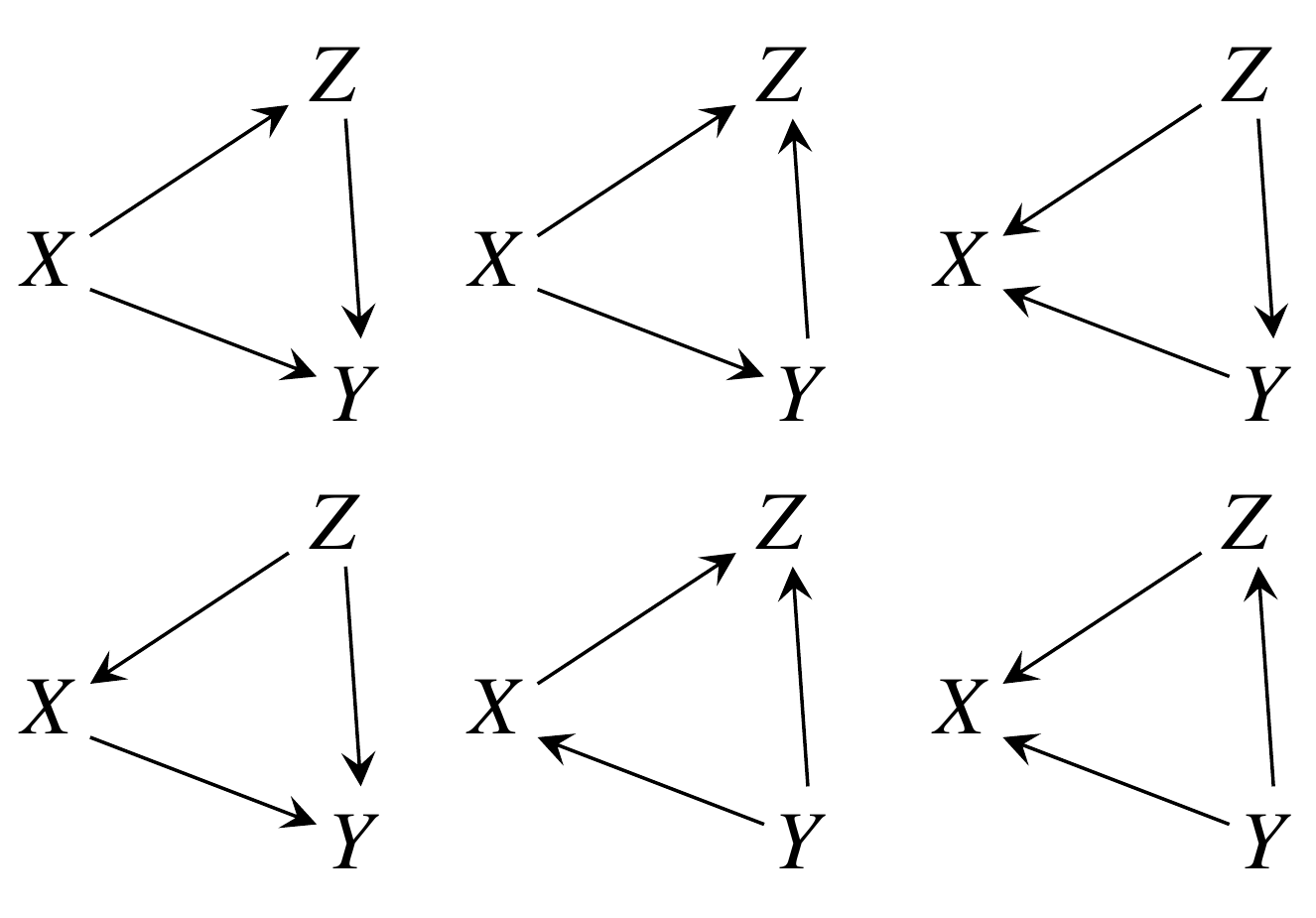}
	\caption{all six possible triangle DAGs on the set of variables $\{X,Y,Z\}$. In the first, second and the third column, variables $Y$, $Z$ and $X$ is the sink variable, respectively.}
	\label{fig:tri}
\end{figure}

In this subsection we will show that under certain conditions, one can still categorize, or in some cases, even uniquely identify a triangle DAG using interaction information. We first observe the following property regarding the triangle DAGs:
\begin{lemma}
\label{lem:tri}
Any triangle DAG contains
\begin{itemize}
\item \text{Root Variable:} Denoted by $R$, the variable which is the cause of the other two variables.
\item \text{Sink Variable:} Denoted by $S$, the variable which is the effect of the other two variables.
\item \text{Bridge Variable:} Denoted by $B$, the variable which is the effect of the root variable and the cause of the sink variables.
\end{itemize}
\end{lemma}
\begin{proof}
Consider fixed labeling on the vertices. Since the graph is acyclic, either two of the arrows are oriented clockwise and the third one is oriented counter clockwise, or vice versa.\\
In either case, two consecutive arrows will have the same direction. The vertex at the tail of the first arrow will be the root variable, the vertex at the head of the first arrow will be the bridge variable, and the vertex at the head of the second arrow will be the sink variable.\\
\end{proof}

Our extra requirement for distinguishing triangle DAGs is for the causal influence with the least strength to be \textit{weak}. Weak here means the causal strength is less than the absolute value of the interaction information among the three variables. Denoting the strength of a causal influence by $\mathcal{C}$, we require that $\mathcal{C}_{min}<|I(R;B;S)|$. Here, 
as mentioned in Subsection \ref{subsec:pre}, 
we need to be able to quantify the strength of a causal effect, for which we use the results from  \cite{janzing2013quantifying}, which was described in Subsection \ref{subsec:pre}.
The postulated quantity in \cite{janzing2013quantifying} for causal influence strength implies that
\begin{equation}
\label{eq:strength}
\begin{aligned}
&\mathcal{C}_{R\rightarrow B}=I(R;B)\\
&\mathcal{C}_{R\rightarrow S}\ge I(R;S|B)\\
&\mathcal{C}_{B\rightarrow S}\ge I(B;S|R)
\end{aligned}
\end{equation}

\begin{figure}[t]
	\centering
	\includegraphics[scale=0.58]{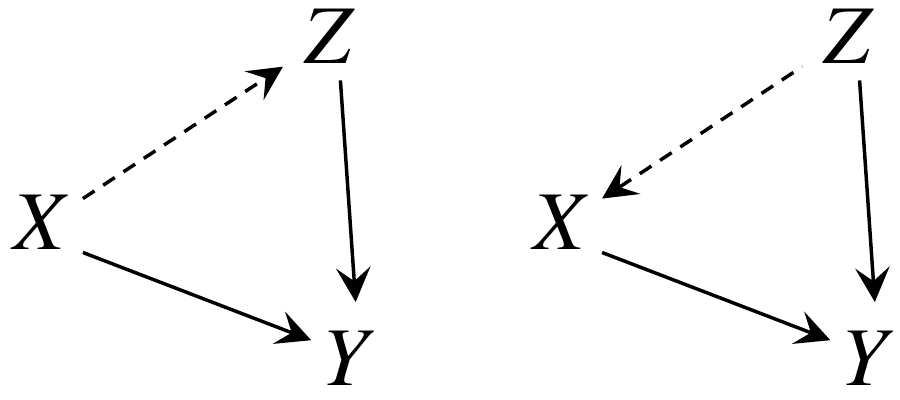}
	\caption{Solid arrows indicate strong causal influence and dashed arrows indicate weak causal influence. This figure shows possible triangle DAGs for the case that the interaction information among the variables is negative, and the causal influence between variables $X$ and $Z$ is weak.}
	\label{fig:cor}
\end{figure}

\begin{theorem}
\label{thm:main}
If in a triangle DAG the interaction information among the variables is negative, and $\mathcal{C}_{min}<|I(R;B;S)|$, then only the causal influence between the Root variable and the Bridge variable is weak.
\end{theorem}
\begin{proof}
From \eqref{eq:strength}, and Proposition \ref{prop:intinfo}, we have
\begin{equation*}
\begin{aligned}
&\mathcal{C}_{R\rightarrow B}=I(R;B)<I(R;B|S)\\
&I(R;S)<I(R;S|B)\le\mathcal{C}_{R\rightarrow S}\\
&I(B;S)<I(B;S|R)\le\mathcal{C}_{B\rightarrow S}.
\end{aligned}
\end{equation*}
Therefore by non-negativity of the mutual information, we have \[\mathcal{C}_{R\rightarrow S}\ge |I(R;B;S)|\] and \[\mathcal{C}_{B\rightarrow S}\ge |I(R;B;S)|.\]Therefore, since $\mathcal{C}_{min}<|I(R;B;S)|$, we must have \[\mathcal{C}_{R\rightarrow B}<|I(R;B;S)|.\]
\end{proof}

Theorem \ref{thm:main} can be utilized in application using the following corollary:
\begin{corollary}
\label{cor:m1}
If in a triangle DAG on variables $X$, $Y$ and $Z$ the interaction information among the variables is negative, and the causal influence between variables $X$ and $Z$ is weak, then
\begin{enumerate}
\item the other two causal influences are not weak,
\item and the only possible triangle DAGs on $\{X,Y,Z\}$ are the ones depicted in Figure \ref{fig:cor}.
\end{enumerate}
\end{corollary}

\begin{figure}[t]
	\centering
	\includegraphics[scale=0.6]{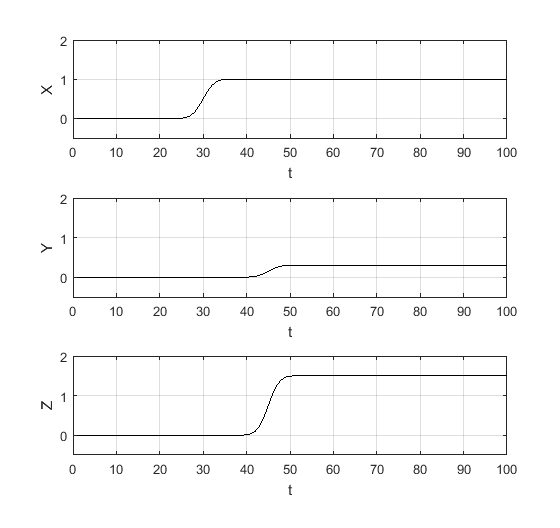}
	\caption{Sample paths of the values of three variables in a triangle structure. Here the temporal information in the sample paths, can help the experimenter to distinguish the cause from the effect.}
	\label{fig:ex2}
\end{figure}

In some applications, due to prior knowledge about the system or temporal knowledge, the root variable is known.
For instance, there is an attribute that the experimenter is randomizing as the root variable in a triangle structure; or by observing sample paths similar to the one shown in Figure \ref{fig:ex2} on a triangle structure, due to temporal information, the root variable could be recognized. That is, the experimenter observes that changing in the value of variable $X$ causes variation in the value of variables $Y$ and $Z$ from the delay, while changing in the values of variables $Y$ and $Z$ do not vary the value of $X$.

In this case, the following corollary of Theorem \ref{thm:main} can be used to uniquely identify the true underlying causal DAG.

\begin{corollary}
\label{cor:m2}
If in a triangle DAG on variables $\{X,Y,Z\}$,
the interaction information among the variables is negative,
and the causal influence between variables $X$ and $Z$ is weak and $X$ is the root variable,
then $Z$ is the bridge variable and $Y$ is the sink variable,
and the only correct causal network among the variables is the DAG on the left side in Figure \ref{fig:cor}.
\end{corollary}
An example of the application of Corollary \ref{cor:m2}, would be in a medical examination, in which the clinician is aware of two side-effects of the medicine which is being tested, say, headache and insomnia, but the side effects themselves have influence on each other, and the direction of this influence is of interest.

\section{Conclusion}
\label{sec:conc}

We studied interaction information, which is a multivariate generalization of mutual information and indicates the amount of information shared in a set of variables, beyond the information which is shared in any proper subset of those variables. Unlike other conventional measures of information, interaction information can have a negative value. We used this property to discover causal relationships among a triplet of random variables.
We provided a discussion regarding the sign of interaction information and proposed a strategy for classifying causal relationships, which could have not been identified using merely conventional dependency tests.
Interaction information is not as thoroughly studied as its bivariate counterpart. A more comprehensive study of the advantages of this quantity in the field of causal inference, especially in the case of having more than three variables is considered as our future work.

\bibliographystyle{ieeetr}
\bibliography{ref}

\begin{thebibliography}{10}

\bibitem{watanabe1960information}
S.~Watanabe, ``Information theoretical analysis of multivariate correlation,''
  {\em IBM Journal of research and development}, vol.~4, no.~1, pp.~66--82,
  1960.

\bibitem{studeny1998multiinformation}
M.~Studen{\`y} and J.~Vejnarov{\'a}, ``The multiinformation function as a tool
  for measuring stochastic dependence,'' in {\em Learning in graphical models},
  pp.~261--297, Springer, 1998.

\bibitem{mcgill1954multivariate}
W.~J. McGill, ``Multivariate information transmission,'' {\em Psychometrika},
  vol.~19, no.~2, pp.~97--116, 1954.

\bibitem{fano1961book}
R.~M. Fano, {\em The Transmission of Information: A Statistical Theory of
  Communication}.
\newblock MIT Press, Cambridge, Massachussets, 1961.

\bibitem{yeung1991new}
R.~W. Yeung, ``A new outlook on shannon's information measures,'' {\em IEEE
  transactions on information theory}, vol.~37, no.~3, pp.~466--474, 1991.

\bibitem{bell2003co}
A.~J. Bell, ``The co-information lattice,'' in {\em Proceedings of the Fifth
  International Workshop on Independent Component Analysis and Blind Signal
  Separation: ICA}, vol.~2003, Citeseer, 2003.

\bibitem{jakulin2003quantifying}
A.~Jakulin and I.~Bratko, ``Quantifying and visualizing attribute
  interactions,'' {\em arXiv preprint cs/0308002}, 2003.

\bibitem{quinn2013efficient}
C.~J. Quinn, N.~Kiyavash, and T.~P. Coleman, ``Efficient methods to compute
  optimal tree approximations of directed information graphs,'' {\em IEEE
  Transactions on Signal Processing}, vol.~61, no.~12, pp.~3173--3182, 2013.

\bibitem{etesami2016learning}
J.~Etesami, N.~Kiyavash, K.~Zhang, and K.~Singhal, ``Learning network of
  multivariate hawkes processes: A time series approach,'' {\em arXiv preprint
  arXiv:1603.04319}, 2016.

\bibitem{quinn2011equivalence}
C.~J. Quinn, N.~Kiyavash, and T.~P. Coleman, ``Equivalence between minimal
  generative model graphs and directed information graphs,'' in {\em
  Information Theory Proceedings (ISIT), 2011 IEEE International Symposium on},
  pp.~293--297, IEEE, 2011.

\bibitem{kocaoglu2016entropic}
M.~Kocaoglu, A.~G. Dimakis, S.~Vishwanath, and B.~Hassibi, ``Entropic causal
  inference,'' {\em arXiv preprint arXiv:1611.04035}, 2016.

\bibitem{etesami2016learning2}
J.~Etesami, N.~Kiyavash, and T.~Coleman, ``Learning minimal latent directed
  information polytrees,'' {\em Neural Computation}, 2016.

\bibitem{pearl2009causality}
J.~Pearl, {\em Causality}.
\newblock Cambridge university press, 2009.

\bibitem{tsujishita1995triple}
T.~Tsujishita, ``On triple mutual information,'' {\em Advances in applied
  mathematics}, vol.~16, no.~3, pp.~269--274, 1995.

\bibitem{koller2009probabilistic}
D.~Koller and N.~Friedman, {\em Probabilistic graphical models: principles and
  techniques}.
\newblock MIT press, 2009.

\bibitem{spirtes2000causation}
P.~Spirtes, C.~N. Glymour, and R.~Scheines, {\em Causation, prediction, and
  search}.
\newblock MIT press, 2000.

\bibitem{janzing2013quantifying}
D.~Janzing, D.~Balduzzi, M.~Grosse-Wentrup, B.~Sch{\"o}lkopf, {\em et~al.},
  ``Quantifying causal influences,'' {\em The Annals of Statistics}, vol.~41,
  no.~5, pp.~2324--2358, 2013.

\end{thebibliography}

\end{document}